\documentclass{article}

\usepackage{PRIMEarxiv}

\usepackage[utf8]{inputenc} 
\usepackage[T1]{fontenc}    
\usepackage{hyperref}       
\usepackage{url}            
\usepackage{booktabs}       
\usepackage{amsfonts}       
\usepackage{nicefrac}       
\usepackage{microtype}      
\usepackage{lipsum}
\usepackage{fancyhdr}       
\usepackage{graphicx}       
\graphicspath{{media/}}     
\usepackage{algorithm}
\usepackage{amsmath}
\usepackage{amssymb}
\usepackage{mathtools}
\usepackage{amsthm}
\usepackage{graphicx}
\usepackage{subfigure}
\usepackage{booktabs}
\usepackage{tikz} 
\usepackage{microtype}
\usepackage{bbm}
\usepackage{graphicx}
\usepackage{subfigure}
\usepackage{booktabs}
\usepackage{amsmath}
\usepackage{bm}
\usepackage{amssymb}
\usepackage{mathtools}
\usepackage{amsthm}
\usepackage{algorithm}%
\usepackage{algpseudocode}
\usepackage{comment}
\usepackage[american]{babel}
 
\usepackage{bibentry}
\theoremstyle{plain}
\newtheorem{theorem}{Theorem}[section]

\theoremstyle{definition}

\theoremstyle{remark}

\usepackage{newfloat}
\usepackage{listings}
\title{Provably Invincible Adversarial Attacks on Reinforcement Learning Systems: A Rate-Distortion Information-Theoretic Approach
}

\author{
  Ziqing Lu \\
  University of Iowa \\
  \texttt{ziqlu@uiowa.edu}
  \And
  Lifeng Lai \\
  University of California, Davis \\
  \texttt{lflai@ucdavis.edu}
  \And
  Weiyu Xu \\
  University of Iowa \\
  \texttt{weiyu-xu@uiowa.edu}
}

\begin{document}
\maketitle

\begin{abstract}
Reinforcement learning (RL) for the Markov Decision Process (MDP) has emerged in many security-related applications, such as autonomous driving, financial decisions, and drone/robot algorithms. In order to improve the robustness/defense of RL systems against adversaries, studying various adversarial attacks on RL systems is very important. Most previous work considered deterministic adversarial attack strategies in MDP, which the recipient (victim) agent can defeat by reversing the deterministic attacks. In this paper, we propose a provably ``invincible'' or ``uncounterable'' type of adversarial attack on RL. The attackers apply a rate-distortion information-theoretic approach to randomly change agents' observations of the transition kernel (or other properties) so that the agent gains zero or very limited information about the ground-truth kernel (or other properties) during the training. We derive an information-theoretic lower bound on the recipient agent's reward regret and show the impact of rate-distortion attacks on state-of-the-art model-based and model-free algorithms. We also extend this notion of an information-theoretic approach to other types of adversarial attack, such as state observation attacks.
\end{abstract}


\section{Introduction}
Reinforcement learning (RL) algorithms are used in many safety or security-related applications, such as autonomous driving \cite{ad}, financial decisions \cite{fd}, recommendation systems \cite{rs}, wireless communication \cite{wireless-commu}, and also drone/robot algorithms \cite{drones}. However, safety concerns remain when RL is deployed in these real-world applications. Studying potential adversarial attacks on RL systems can help evaluate the worst-case performances of RL agents under these attacks. This will in turn help us limit the damage imposed by adversarial parties, defend against adversarial attacks, and therefore build more robust and secure RL systems. The poisoning adversarial attacks on RL can be categorized based on the specific component targeted by the attacker, either on the victim agent itself or the environment. Common types include action attacks, reward attacks, state attacks, environmental attacks, or mixed attacks \cite{att-costsig, davis-ma, ada-rpa,davis-sa, batch-rl, spaold, envattack1,att-costsig, poli-induc}. These attacks either directly change the features of agents such as actions, rewards, and states of the MDP, or perturb the transition dynamics of the environment.

However, most previous works considered \emph{deterministic} attacks on the state transitions such that the victim agent will observe a false but deterministic transition kernel. Unfortunately, these types of deterministic attacks often assume that the victim agent is unaware of the existence of attacks. In fact, if the victim is aware of the attacks, it can potentially find a mapping between the learned false transition kernel and the ground-truth kernel and then reversely recover the ground-truth kernel from the false learned kernel. This will potentially make adversarial attacks ineffective. Moreover, these adversarial attacks cannot be guaranteed to work or be proven to be ``invincible'': it is not clear whether these attacks are effective regardless of what defenses the victim agent might use.

Unlike previous works, we propose a provably ``invincible'' poisoning adversarial attack. This adversarial attack fools the victim agent into observing a \emph{random} false state transition kernel instead of a deterministic false transition kernel. Since the transition kernel itself records transition probability distributions between states, this new attack imposes a probability distribution over the state transition probability kernel, namely, it introduces a probability distribution over ``(transition) probability distributions''. Through this attack, the observed transition kernel that the victim agent learns in RL training will provably offer zero or little information about the ground-truth kernel, guaranteed through information-theoretic analysis. This attack is ``invincible'' because even if the victim agent knows the attacker's strategy, the victim agent does not know exactly what specific ground-truth transition kernel is. This in turn makes the victim agent experience an inevitable regret in the obtained reward due to its uncertainty about the ground-truth transition kernel. 

The motivation and modeling of such attacks on MDPs with random kernels can be explained through an air defense application example. Consider two opposing agents, $A$ and $B$. $B$ uses probing-drones to determine whether $A$ has an unlimited-resource air defense system. During the training stage, $B$ uses probing-drones to learn about $A$'s defense deployment, while in the test stage, $B$ launches a large-scale attack-drone attack. Before training, $B$ knows the prior probability of $A$ having an unlimited-resource air defense system (e.g., $50\%$). Here the number of drones is the state of the MDP. If an air defense system is present, the MDP evolves according to the transition kernel $X_1$; otherwise, it follows $X_2$. $A$ can change the transition kernel in the training stage using an environment hyperparameter: setting up air defense or not for the probing-drones.

In our formulation, $A$ acts as the attacker, while $B$ is the victim. Notice that the term ``attack'' refers to adversarial manipulation of the MDP, not a physical drone strike. $A$ can perform the poisoning attack proposed in our paper as follows: If the ground-truth kernel is $X_1$, meaning Agent $A$ has an unlimited-resource air defense system, $A$ will use the defense with a probability of $0.5$ during the training stage. If $A$ chooses to use air defense, it will do so consistently throughout training; otherwise, it will not use air defense at all. If there is no unlimited-resource air defense system, $A$ may borrow one (or deplete its limited resources) with a probability $0.5$. As a result, $B$ will calculate a posterior probability of $0.5$ for both $X_1$ and $X_2$, regardless of whether $A$ has an unlimited-resource air defense or not. Thus, $B$ gains no useful information about the true transition kernel, and $A$'s adversarial manipulation is successful.

In this paper, we provide a lower bound on the expected regret experienced by the victim agent, regardless of the defense strategies the victim adopts, through an information-theoretic analysis. To the best of our knowledge, this is among the first works to formulate adversarial attacks on RL algorithms using rate-distortion theory from the information theory perspective. 
In addition, we introduce a theoretical analysis of the existence of an optimal policy for MDPs with uncertain transition kernels. We propose a new policy iteration algorithm to find the optimal policy for MDP with uncertain transition kernels. Our numerical and theoretical results show that the rate-distortion information-theoretic adversarial attack can greatly reduce the victim agent's expected reward compared to deterministic attacks.

In this paper, we focus on adversarial attacks in which we adopt an information-theoretic rate-distortion approach to change the ground-truth transition kernels to random rather than deterministic delusional transition kernels. However, this rate-distortion approach is very general and can be applied to other types of adversarial attacks, for example, state perception attacks, action attacks and reward attacks mentioned above. In applying this rate-distortion idea to these attacks, we randomly rather than deterministically map the ground-truth quantities or actions to delusional ones. In fact, \emph{random state perception attacks} can be used to achieve random-transition-kernel attacks discussed above: Please see Section \ref{sec:formulation}, \emph{Attack procedure}.

\textbf{Related works}: In a nutshell, we propose a novel adversarial attack by injecting statistical uncertainties into the transition kernel, as inferred by the victim agent. We discuss related works from two aspects: (1) uncertainty of transition kernels, as studied in robust and Bayesian RL, and (2) adversarial attacks targeting transition functions: 

(1) \textbf{Uncertainty of transition kernels.} In distributionally-robust Markov decision processes (DR-MDP), agents seek to optimize their performance under the worst-case transition kernel within an ambiguity set around the nominal transition probability. \cite{iyengar2005robust,nilim2003robustness, pinto2017robust}. \cite{xu2010distributionally} extends the approach to DR-MDP with parameter uncertainty and uses the probabilistic information of the unknown parameters.   
However, our work is different from the robust RL literature in that we assume the uncertainty of the transition function from a Bayesian perspective. There is a prior probability distribution over the transition kernels, instead of a fixed small uncertainty set for the transition kernel. 

From a statistical perspective, the closest related work is the Bayesian adaptive MDP (BA-MDP) framework \cite{duff2002optimal, ghavamzadeh2016bayesian, lin2022bayesian}. BA-MDP assumes a prior distribution over problem instances and updates the posterior according to the Bayes rule as data are collected. In BA-MDPs, uncertainty over the transition function is typically parameterized by a set of parameters $\bm{\Phi}$. Starting with a prior belief about $\bm{\Phi}$, the agent collects data and uses Bayes' rule to compute the posterior distribution of the parameters.  A traditional approach to solving such problems involves belief tracking operation: keeping a joint posterior distribution over model parameters and the true physical state of the MDP, and deriving a policy that selects optimal actions with respect to the both the state and the posterior distribution. In contrast, we assume that 
the victim agent's belief over transition kernels remains ``ambiguous''  during the whole training, instead of BA-MDP's concentrating to the ground-truth kernel as the number of time steps goes to infinity. This is because under our adversarial attacks, the ground-truth transition kernels are still uncertain or ambiguous to the victim agent with its observations corrupted or attacked. Therefore, the distribution over transition kernels should remain ``ambiguous'' throughout the training. Most importantly, none of the previous works listed above considered an adversarial attack setting. 

(2) \textbf{Poisoning attacks on transition kernels}: \cite{envattack1} attacks the victim agents by manipulating the environment hyper-parameters corresponding to the dynamics of the training environment. In \cite{rakhsha2020policyteachingenvironmentpoisoning}, the authors propose an offline poisoning attack such that the attacker constructs an optimal fake transition kernel according to a target policy and places it in the victim's planning environment. \cite{envattack1} considers an adapative attack according to the victim's policy, and \cite{rakhsha2020policyteachingenvironmentpoisoning} changes the transition kernel deterministically. However, our proposed attack changes the transition kernel statistically from a designed distribution. \cite{franzmeyer2024illusory} integrates information-theoretic detectability constraints to create stealthy adversarial attacks. However, they use information theory to constrain the statistical detectability of attacks during test-time, by bounding the KL-divergence between the attacked and original observation trajectories. In contrast, our rate-distortion attack focuses on attacking the transition dynamics randomly during the training, aiming to minimize the mutual information between the agent's observed and ground-truth environment. \cite{game-theoretic} formulates robust RL as a constrained
minimax game between the RL agent and an environmental adversary that represents uncertainties such
as model parameter variations and adversarial disturbances.  In \cite{game-theoretic}, the adversary environmental agent picks an optimal policy parameter which generates a distribution over the environment uncertainty parameter, assuming that the victim agent chooses an optimal policy of its own. In contrast, the adversary considered in our paper 
statistically changes the victim agent's perception of the MDP's underlying transition kernel into false transition kernels following a ground-truth-kernel-dependent distribution. 
In addition, different from all of the above, we consider the Bayesian setting for the ground-truth transition kernel, where there is a prior probability distribution over possible ground-truth kernels. Also note that our attack works in an offline setting, where the attackers first statistically change the transition kernel, after which the victim RL agent finds its policy given the observed transition kernel.

\section{Problem Formulation}
\label{sec:formulation}

We consider an MDP where its state transition kernel is random and has a prior probability distribution $p$ over its sample space. Once the ground-truth transition kernel is sampled from this prior probability, it will remain the same during both training and testing stages, for every time step and every episode. An adversarial attacker tries to perform attacks on the transition kernel so that the victim agent will learn a false or ambiguous transition kernel during training. 

An MDP with the random transition kernel can be represented by a five-element tuple: $(\mathcal{S},\mathcal{A},X,r,\gamma)$, where $\mathcal{S}$ is the state space with $|\mathcal{S}|=S$. $\mathcal{A}$ is the action space with $|\mathcal{A}|=A$. $X:\mathcal{S}\times\mathcal{A}\times\mathcal{S}$ is the discrete random transition kernel sampled from a prior distribution $p$. $p$ is a discrete prior distribution. $r:\mathcal{S}\times\mathcal{A}\times\mathcal{S}\rightarrow[0,1]$ is the reward function and $\gamma\in [0,1]$ is the discount factor. 
We use $\pi:X\times\mathcal{S}\rightarrow \mathcal{A}$ to denote the deterministic policy of the victim agent. We let $\Omega(X)$ be the sample space of the random transition kernel $X$, and, for the sake of easy presentation, we assume that there is a finite number ($|\Omega(X)|$) of possible transition kernels. 

In the following, we use the notation $\pi(X_i)$ to denote a policy for the fixed transition kernel $X_i$, where $X_i\in\Omega(X), i=1,2,\dots, |\Omega(X)|$. We use the notation $\pi(X_i)(s)$ to denote the policy $\pi(X_i)$ maps from state $s\in\mathcal{S}$ to some action $a\in\mathcal{A}$. As standard in the RL literature, we let the V-value be a state value function that records the expected reward associated with every state $s\in\mathcal{S}$. We denote the V-value for an agent following an arbitrary $\pi$ under a fixed transition kernel $X_i$  as $V^{\pi}_{X_i}(s)=\mathbb{E}_{X_i}[\sum_{k=0}^{\infty} \gamma^kr(s_k,a_k,s_{k+1})|a_k\sim\pi,s_0=s]$, 
For every fixed ground-truth transition kernel $X_i$, the optimal policy corresponding to $X_i$ is defined as $\pi^*(X_i)=\arg\max_{\pi}V^{\pi}_{X_i}$. 

We consider an adversarial attack setting in which the attackers can change the agent's perception of the probability transition kernel $X$ to a random delusional kernel $Y$ from the sample space $\Omega(Y)$. For the sake of presentation, we consider $\Omega(Y)=\Omega(X)$, but they can be different.

In previous work on RL adversarial attacks, $X$ is often changed to a different but deterministic transition kernel $Y$ \cite{rakhsha2020policyteachingenvironmentpoisoning}. However, the victim agent can potentially reversely recover the ground-truth $X$, and there is no guarantee that such an adversarial attack is effective. We instead propose a rate-distortion attack, where we statistically change the victim agent's perception of $X$ into random $Y$, during the RL training.  

\textbf{Attack setting.} Suppose that the ground-truth probability transition matrix $X$ follows a prior distribution $p(X)$, the attackers design a rate-distortion attack by setting the likelihood $P(Y|X)$, which is the probability of the victim agent seeing the delusional transition kernel $Y$ given the ground-truth kernel is $X$. An illustration graph of the rate-distortion attack is in Figure \ref{rd-attack} for $3$ possible ground-truth kernels and 3 possible delusional perceived kernels. In this example, if the ground-truth kernel is $X_1$, we attack the RL system such that the victim agent sees delusional kernels $Y_1$, $Y_2$, and $Y_3$ with probability $0.1$, $0.2$, and $0.7$ respectively.

There is a cost, denoted by $C(X\rightarrow Y)$ associated with changing $X$ to perception $Y$, and the attackers are constrained by an attack budget $B$ on the expected cost.  The attacker's goal is to change $X$ to perception $Y$ randomly so that they follow a joint probability distribution $p(X,Y)$ satisfying the following conditions:

(1) The average cost of changing from $X$ to $Y$ is less than some budget $B$. This constraint can be transformed into a constraint on $p(X,Y)$. 

(2) We assume that the victim agent is aware of the existence of the attack. The agent knows the posterior conditional probability $p(X|Y)$ and the joint probability $p(X,Y)$.  The victim agent will try to find a policy given that its perception of the transition kernel is $Y$. 

To maximize its expected reward, the victim agent derives the optimal policy $\pi^*(\cdot|Y)$ under the posterior probability $p(X|Y)$. For a fixed delusional kernel $Y_i$, the optimal policy is chosen to be $$\pi^*(\cdot|Y_i)=\arg \min_{\pi} \mathbb{E}_{P(X|Y_i)} \|V^{\pi}_X-V^{\pi^*(X)}_X\|_\infty, $$ where the goal is to minimize the \emph{regret} in V-value compared with the optimal policy for the ground-truth transition kernel.  We define the attacked V-value as $$V^{\pi^*(\cdot|Y_i)}= \mathbb{E}_{P(X|Y_i)} [V^{\pi^*(\cdot|Y_i)}_X].$$
\begin{figure}
    \centering
    \includegraphics[scale=0.15]{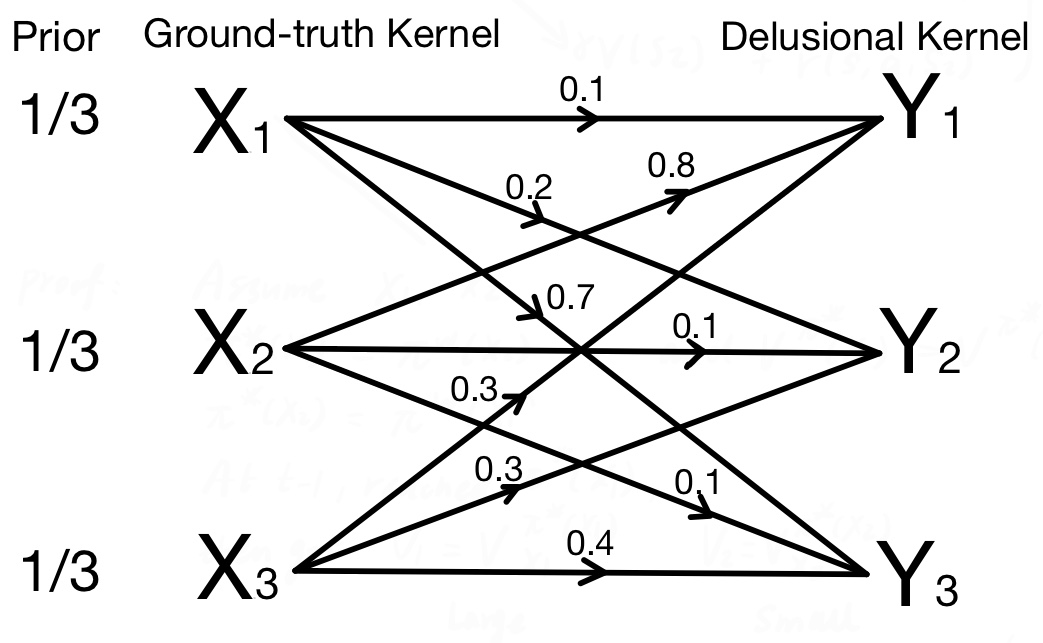}
    \caption{Rate distortion attack}
    \label{rd-attack}
\end{figure}
We now define the regret of the reward the victim agent receives due to this adversarial attack. 

\textbf{Definition:} (Regret) Let $\pi^*(X)$ denote the optimal policy the victim agent uses when the ground-truth transition kernel is $X$, and let $\pi^*(\cdot|Y)$ be the optimal policy when the agent sees the delusional kernel $Y$. We define the total regret of the victim agent as
$$R=\mathbb{E}_{p(X,Y)}\|V^{\pi^*(X)}-V^{\pi^*(\cdot|Y)}\|_{\infty}.$$\vspace{-0.1cm}The attacker's goal is to design $p(X, Y)$ such that regret $R$ is maximized, while the victim agent's goal is to minimize regret while having access to the probability distribution $p(X, Y)$ and observation $Y$.  

We propose a rate-distortion adversarial attack on the recipient agent, which we can prove to be effective no matter what defense the recipient agent uses. In this attack, the attacker solves the following optimization question: 
\begin{align}
\min_{p(X, Y)} \quad &  I(X;Y)\\
\textrm{s.t.} \quad &   \mathbb{E}_{p(X,Y)} C(X\rightarrow Y)\leq B, 
\end{align}
where $I(X;Y)$ is the mutual information between $X$ and $Y$, and $B$ is the attack budget.  

\textbf{Attack Procedures.} We describe the procedures of attackers performing rate-distortion attacks on the victim RL agents in training. The environment first samples a ground truth transition kernel $X$. The attacker then samples a delusional kernel $Y$ from a designed distribution $P(Y|X)$ before the RL agent begins training. To change the transition kernel from $X$ to $Y$, the attacker can either modify a hyper-parameter that controls the dynamics of the training environment \cite{envattack1}, design a fake kernel in the same form similar to the ground truth $X$ and replace $X$ directly with $Y$ \cite{rakhsha2020policyteachingenvironmentpoisoning}, or change state observations \cite{DBLP:journals/corr/abs-1712-03632}. Thus, the victim agent is trained in a poisoned environment with the fake transition kernel $Y$, and obtains the corresponding posterior distribution $P(X|Y)$. In model-based learning, we assume that the RL agent perceives the fake kernel $Y$ and uses planning algorithms to find the corresponding optimal policy $\pi^*(\cdot|Y)$. In model-free learning, we assume that each fixed ground-truth kernel $X$ corresponds to a different optimal policy $\pi^*(X)$, and the victim agent directly learns a fake optimal policy $\pi^*(Y)$ by a model-free algorithm such as tabular or deep Q-learning. The victim infers $P(X|Y)$ after the model-free learning from its learned optimal policy 
The expectation over $(X, Y)$ can be seen as taking the average over many ``runs" of the whole process (each run is considered a fresh start of everything).

We emphasize the practical feasibility of the proposed rate-distortion attack on the transition kernel. Except for approaches that modify the hyper-parameters\cite{rakhsha2020policyteachingenvironmentpoisoning} or construct and directly substitute a fake kernel that mimics the true kernel\cite{envattack1}, the rate-distortion attacks can be implemented far more \emph{easily} by perturbing the victim's state observations randomly during the training. For example, consider a model-based learner that is estimating the environment model $X$. During data collection, suppose the ground-truth trajectories are in the form $(s_1,a_1,r_1,s_2,a_2,\dots,s_{T+1})$. An attacker that manipulates state perceptions can make the victim observe the corrupted trajectory $(s^d_1,a_1,r_1, s_2^d,a_2,\dots,s_{T+1}^d)$, where each $s_i^d$ denotes the delusional state perceived by the victim agent after the attack. For a finite, discrete state space $\mathcal{S}$ of cardinality $S$, there are $S!$ possible permutations between states. The attacker selects one permutation uniformly at random and applies it to each observed state. Therefore, the victim sees the fake, permuted transition kernel during the training. This simple random state observation attack therefore realizes a rate-distortion attack on the transition kernel with substantially weak requirements: The attacker neither needs to directly manipulate the transition kernel, nor needs to know the exact values of the transition kernel. Randomness comes from the random permutation of states.

In the next section, we provide an information-theoretic lower bound on the regret that the recipient agent will experience. For theoretical proofs, we assume that the victim uses model-based learning and planning algorithms to find its optimal policy $\pi^*(\cdot|Y)$.

\section{Provable lower bound on regret caused by rate-distortion attack}
In order to provide a lower bound on the regret experienced by the recipient agent, we first introduce the concept of an optimal decoder minimizing the error probability of decoding for $X$ based on a delusional observation $Y$.  

\textbf{Definition:}(Optimal decoder) Let $f:\Omega(Y)\rightarrow \Omega(X)$ be a decoder function of the ground-truth transitional probability $X$, i.e., $\hat{X}=f(Y)$ is an estimation of $X$. Let $g$ denote the optimal decoder such that
$$g=\arg\min_{f}\mathbb{E}_{p(X,Y)}\mathbbm{1}_{X\neq f(Y)}.$$
and let $Pe$ denote the error probability of the optimal decoder function $g$: $Pe = \mathbb{E}_{p(X,Y)}\mathbbm{1}_{X\neq g(Y)}.$

Then we have the following lower bound on the regret.  

\begin{theorem}\label{thm3.1}
    Assume that for each ground-truth $X$, the recipient agent has a different optimal policy $\pi^*(X)$; and there exists a positive constant $\epsilon>0$ such that 
    $\|V^{\pi^*(X)}_{X}-V^\pi_{X}\|_{\infty}>\epsilon$,  $\forall \pi\neq\pi^*(X)$. 
    Then the regret under the attack described by $p(X,Y)$ satisfies $$R\geq \epsilon Pe,$$
    where $Pe$ is defined as the error probability of the optimal decoder $g$. Moreover, 
    $$H(P_e)+P_e\log |\Omega(X)|\geq H(X|Y)=H(X)-I(X;Y),$$
where $H(P_e)= -P_e \log (Pe)-(1-P_e) \log(1-Pe)$ is an entropy, $H(X)$ is the entropy of $X$, and $H(X|Y)$ is the conditional entropy of $X$ given $Y$.

\end{theorem}
\begin{proof}
We construct an estimator $h:\Omega(Y)\rightarrow\Omega(X) $:
\begin{align*}
     h(Y)=\begin{cases}
        \hat{X}, \text{ if }\exists\ \hat{X}\in\Omega{(X)}\text{ such that }\pi^*(\hat{X})=\pi^*(\cdot|Y)\\
        \arg\min_{\hat{X}\in\Omega(X)}\| V^{\pi^*(\hat{X})}-V^{\pi^*(\cdot|Y)}\|_{\infty}, \text{otherwise.}
    \end{cases}
\end{align*}
With the assumption that there is a bijective relation between the ground-truth transition kernel $X$ and the optimal policy $\pi^*(X)$, we can claim that: if the decoded estimation $\hat{X}$ is not equal to the ground-truth transition kernel $X$, i.e., $\hat{X}\neq X$, then the optimal policy for the ground-truth $X$ will be different from the optimal policy $\pi^*(\cdot|Y)$, namely $\pi^*(X)\neq \pi^*(\cdot|Y)$.
Therefore,
$$\mathbbm{1}_{X\neq \hat{X}}\leq \mathbbm{1}_{\pi^*(X)\neq \pi^*(\cdot|Y)},$$
where $\mathbbm{1}$ is the indicator function. 

Then we can give a lower bound of the regret $R$ as follows: 
\begin{align*}  R&=\mathbb{E}_{p(X,Y)}\|V^{\pi^*(X)}-V^{\pi^*(\cdot|Y)}\|_{\infty}\\
&\geq \mathbb{E}_{p(X,Y)}[\mathbbm{1}_{\pi^*(X)\neq \pi^*(\cdot|Y)}\epsilon]\\
&=\epsilon\mathbb{E}_{p(X,Y)}[\mathbbm{1}_{\pi^*(X)\neq \pi^*(\cdot|Y)}]\\
&\geq\epsilon\mathbb{E}_{p(X,Y)}[\mathbbm{1}_{X\neq\hat{X}}]\\
&=\epsilon\mathbb{E}_{p(X,Y)}[\mathbbm{1}_{X\neq h(Y)}]\\
&\geq\epsilon\mathbb{E}_{p(X,Y)}[\mathbbm{1}_{X\neq g(Y)}]\\
&=\epsilon Pe,
\end{align*}
where $g(\cdot)$ is the optimal decoder. 

Furthermore, we make use of Fano's inequality \cite{infobook} from information theory to specifying a lower bound on $P_e$. 
\begin{theorem}[Fano's Inequality]
     For any estimator $\hat{X}$ such that $X\rightarrow Y\rightarrow \hat{X}$, with $P_e=\text{Pr}(X\neq \hat{X})$, we have:
$$H(P_e)+P_e\log |\Omega(X)|\geq H(X|\hat{X})\geq H(X|Y).$$
This inequality can be weakened to 
$1+P_e\log|\Omega(X)|\geq H(X|Y)$, or 
$$P_e\geq \frac{H(X|Y)-1}{\log|\Omega(X)|}.$$
\end{theorem}
Using Fano's inequality together with $R \geq \epsilon P_e$, we conclude the theorem. 
\end{proof}

\section{Optimal policy $\pi^*(\cdot|Y)$ under a random kernel}
\label{headings}
In this rate-distortion adversarial attack, given a delusional transition kernel observation $Y$, the recipient agent is facing a distribution $P(X|Y)$ for the ground-truth kernel. 
The uncertain transition kernel, once realized, remains unchanged over time in our new setting and poses new research challenges. We want to highlight this is different from the Bayesian optimal value and algorithm defined in a belief MDP (BA-MDP). In belief-MDP, the optimal policy plans over both the current state and the belief state, however, in our paper we do not track the belief state because we assume that the agent's belief over transition kernels remains ``ambiguous'' during the whole training. One may wonder how to find the optimal policy $\pi^*(\cdot|Y)$ with the largest V-values under the conditional distribution $P(X|Y)$. Interestingly and counterintuitively, we find that there might not exist policies that maximize the V-values for every state simultaneously.

More specifically, for an MDP with a fixed transition kernel $P$, there exists an optimal deterministic policy $\pi^*$ \cite{sutton-book}. By its optimality, the V values associated with this policy satisfy $V^{\pi^*}(s)\geq V^{\pi}(s),\ \forall \pi\neq\pi^*,\forall s\in\mathcal{S}$. However, we find that for a random transition kernel $X$, if we define the optimal deterministic policy $\pi^*$ as the one that maximizes the expected V-values simultaneously for every $s$, i.e., 
\begin{align}
    \pi^*=\arg_{\pi}\mathbb{E}_XV_X^{\pi}(s), \forall s,\label{optimal solution}
\end{align} 
then such an optimal policy may not exist. 
\begin{theorem}\label{thm4.1}
    Consider an MDP:$(\mathcal{S},\mathcal{A},X,r,\gamma)$, where the transition kernel $X$ is random following the distribution $p(X)$. There does not always exist an optimal policy $\pi^*$ such that $\pi^*=\arg_{\pi}\mathbb{E}_XV^{\pi}_X(s),\forall s\in\mathcal{S}$.
\end{theorem}
\begin{proof}
    Consider an infinite MDP with $\mathcal{S}=\{0,1\}$, $\mathcal{A}=\{a,b\}$, $\gamma=0.9$, and a random kernel $X$ with values for $X_1$ and $X_2$ work as follows:
    \begin{align*}
        X^a_1=\begin{bmatrix}
            0&1\\1&0
        \end{bmatrix},\ \hfill
        X^b_1=\begin{bmatrix}
            1&0\\0&1
        \end{bmatrix},\hfill
        X^a_2=\begin{bmatrix}
            1&0\\0&1
        \end{bmatrix},\ \hfill
        X^b_2 = \begin{bmatrix}
            0&1\\1&0
        \end{bmatrix}.
    \end{align*}
    The prior distribution of $X$ is $p(X_1)=0.5$, and $p(X_2)=0.5$.
    The reward $r(s,s')$ is shown in Table \ref{reward-two-state}.
    
    We can use an exhaustive search to find the optimal policy $\pi^*$ due to the small space in this example. The set of all possible policies is $\Pi=\{\ \pi_1=\{0:a,1:a\}, \pi_2=\{0:a,1:b\},\pi_3=\{0:b,1:a\},\pi_4=\{0:b,1:b\}\}$. To evaluate every policy $\pi\in\Pi$ with a fixed kernel $X_i$, we can apply the Bellman operator:
    \begin{align*}
        V^{\pi}_{X_i}(s)&=\sum_{s'}X_i^{\pi}[r(s,\pi(s),s')+\gamma V^{\pi}(s)]\\
        &= \Bar{r}^{\pi}(s) + \gamma\sum_{s'}X_i^{\pi} V^{\pi}(s) 
    \end{align*}
    Here in the Bellman equation, the notation $X_i^{\pi}$ is defined as $X_i^{\pi}=X_i(s'|s,\pi(s))$,$\forall s$. The notation $\Bar{r}^{\pi}(s)=\sum_{s'}X_i^{\pi}(s'|s,\pi(s))r(s,\pi(s),s')$ depends only on $s$, since $\pi(s)$ is fixed. 

    Therefore, the Bellman equation for $V^{\pi}_{X_i}$ defines a system of 2 linear equations with 2 variables, and the vector form is
    \begin{align*}
        V_{X_i}^{\pi} = \Bar{r}^{\pi} +\gamma X_i^{\pi}V_{X_i}^{\pi}.
    \end{align*}
    Solving the above equation, we get $V_{X_i}^{\pi}(s),\forall s$:
    \begin{align}\label{vector-form} 
       V_{X_i}^{\pi} = (I-\gamma X_i^{\pi})^{-1}\Bar{r}^{\pi}.
    \end{align}
    By computing the V values for each kernel $X_i$ using equation (\ref{vector-form}), we further evaluated the expected V-values $\mathbb{E}_XV^{\pi}_X(s)$ for every state and every policy.
    We see that from Table \ref{expected V} for state $0$, the optimal policy is $=\pi_2\ or \ \pi_3$, but for state $1$, the optimal policy is $\pi_1\ or \ \pi_4$. Notice that an optimal policy $\pi^*$ needs to satisfy $E_XV_X^{\pi^*}(s)>E_XV_X^{\pi}(s), \forall \pi$ for all states simultaneously. Therefore, this is a counterexample showing that there exists an MDP with a random kernel that does not have an optimal policy. Therefore, we proved Theorem $\ref{thm4.1}$.
    
\begin{table}[h]
    \centering
\centering
        \begin{tabular}{|c||c|c|}
        \hline
             $s\downarrow s'\rightarrow$&0&1  \\
             \hline
             0&0.06&0.15\\
             \hline
             1&0.3&0.95\\
             \hline
        \end{tabular}
        \caption{Reward function of the two-state example}\label{reward-two-state}
            \centering
        \begin{tabular}{|c||c|c|c|c|}
            \hline
             state $\downarrow$ policy$\rightarrow$& $\pi_1$&$\pi_2$&$\pi_3$&$\pi_4$ \\
             \hline
             0&1.41&4.65&4.65&1.41 \\
             \hline
             1&5.89&5.17&5.17&5.89\\
             \hline
        \end{tabular}
    \caption{Expected reward over transition kernels: $\mathbb{E}_XV^{\pi}_X(s)$}
        \label{expected V}
\end{table}

\textbf{Non-existence of the optimal \emph{random} policy}: To further investigate whether an optimal \emph{random} policy $\pi^*$ exists under the definition (\ref{optimal solution}), which is, $$\pi^*=\arg_{\pi}\mathbb{E}_XV_X^{\pi}(s),\forall s.$$ 
Now the optimal policy is in the form $\pi^*:\mathcal{S}\times\mathcal{A}\rightarrow [0,1]$. We conduct a numerical experiment under the same setting as described in the above proof.

The results indicate that there \emph{does not} always exist an optimal random policy $\pi^*$ satisfying $\pi^*=\arg_{\pi}\mathbb{E}_XV^{\pi}_X(s),\forall s\in\mathcal{S}$. To visualize this, we make a contour plot over the two probabilistic parameters $\theta_0$ and $\theta_1$, where $\theta_0=\pi(a|0)$ denotes the probability of taking action $a$ in state $0$, and $\theta_1=\pi(a|1)$ denotes the probability of taking action $a$ in state $1$. Consequently, the probabilities of taking action $b$ are $1-\theta_0$ and $1-\theta_1$ in state $0$ and $1$, respectively.

The two subplots in Figure \ref{fig:non-exist-random-policy} illustrates the expected state value surfaces $\mathbb{E}_X V_X^{\pi}(0)$ and $\mathbb{E}_X V_X^{\pi}(1)$ as functions of $(\theta_0, \theta_1)$. In both subplots, the lighter the color is, the larger the expected state value is. We evaluate four discrete priors over the two possible transition kernels, including $p=\{[0,1],[1,0],[0.3,0.7],[0.5,0.5]\}$, and we mark the optimal random policies for each state with red markers in both subplots. The contour plots corresponds to the uniform prior $p=[0.5,0.5]$. In this case, the figure shows that the optimal value for state $0$ and $1$ are at two different policies which are governed by $(\theta_0,\theta_1)=(1,0.37)$ for state $0$ and $(\theta_0, \theta_1)=(0,0)$ for state $1$. There does not exist a common policy that optimizes the expected state value of state $0$ and $1$ simultaneously by looking at the trends of the contours. An optimal random policy exists only when $\pi^*(\cdot|0)=\pi^*(\cdot|1)$, for any action. In other words, the red markers in both subplots with the same prior label should be at the same position. Therefore, there does not exist an optimal random policy under (\ref{optimal solution})

We can observe that (\ref{optimal solution}) only holds in the two deterministic cases for $p=[0,1]$ and $p=[1,0]$. When $p=[0.3,0.7]$ and $p=[0.5,0.5]$, the optimal random policy does not exist.  
\begin{figure}
    \centering
    \includegraphics[scale=0.5]{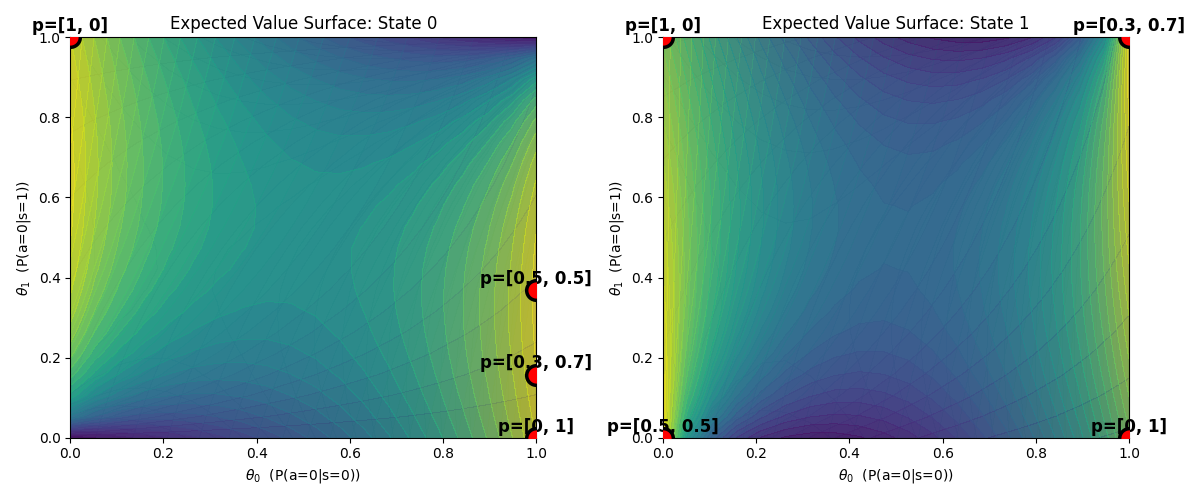}
    \caption{Non-existence of the optimal random policy}
    \label{fig:non-exist-random-policy}
\end{figure}
\end{proof}

\section{Planning with random kernels}
When given a delusional transition kernel $Y$, one is faced with the task of finding an optimal policy $\pi^*(\cdot| Y)$ such that the policy provides the smallest regret in V values, namely finding a policy $\pi$ minimizing 
 $$\mathbb{E}_{P(X|Y)} \|V^{\pi}_X-V^{\pi^*(X)}_X\|_{\infty} .$$

We therefore propose a new policy iteration algorithm described in Algorithm \ref{policy evaluation} and \ref{policy iteration} to find such a policy. In this policy iteration algorithm, in each iteration, we find whether there is an action update that improves the expected V-values over the distribution of the transition kernels and if there is, we update the action to that new action.

In practice, we find that this policy iteration algorithm works well and can find the optimal or near-optimal policy with high probability. However, counterintuitively, different from the usual policy iteration algorithm which provably converges to the optimal policy under a fixed transition kernel,  we can show that this natural extension of policy iteration does not always converge to the optimal policy under random kernels.  

In fact, we find that even if an MDP with a random kernel $X$ indeed has an optimal policy $\pi^*$ as defined in (\ref{optimal solution}), the policy iteration algorithm, which is a commonly used algorithm for planning, does not always converge to the optimal solution $\pi^*$ under random transition kernels.

\begin{theorem}\label{thm5.1}
    If a given MDP represented by $(\mathcal{S},\mathcal{A},X,r,\gamma)$, with a random transition kernel $X$ distributed from $p(X)$, has an optimal policy $\pi^*=\arg_{\pi}\mathbb{E}_{X}V^{\pi}_X$, the extended policy iteration algorithm \ref{policy evaluation} and \ref{policy iteration} does not always converge to $\pi^*$. 
\end{theorem}

\begin{algorithm}
  \caption{Policy Evaluation}\label{policy evaluation}
    \textbf{Input}: $\mathcal{S}, \mathcal{A},\Omega(X),r,\gamma,\pi$
  \begin{algorithmic}[1]
  \For{$X_i\in\Omega(X)$} 
  \State Repeat
  \State $\Delta_{X_i}\gets 0$
\For{$s\in\mathcal{S}$}
\State $v_{X_i}\gets V_{X_i}^{\pi}(s)$
        \State $V_{X_i}^{\pi}(s)\gets\sum_{s'}X_i(s'|s,\pi(s))[r(s,\pi(s),s')+\gamma V^{\pi}_{X_i}(s')]$
        \State $\Delta_{X_i} \gets\max(v_{X_i}-V^{\pi}_{X_i}(s))$
 \EndFor 
 \State until $\Delta_{X_i} < \theta$\Comment{$\theta$ is a small threshold}
 \EndFor
 \State \Return $V^{\pi}_{X_i}, X_i\in\Omega(X_i)$
  \end{algorithmic}
\end{algorithm}
\begin{algorithm}
\caption{Policy iteration with a random kernel}\label{policy iteration}
\textbf{Input}: $\mathcal{S}, \mathcal{A},V_{X_i}^{\pi},\Omega(X),r,\gamma,\pi,p,n$
    \begin{algorithmic}[1]
    
    \For{$k$ from $1$ to $n-1$}
    \For{$s\in\mathcal{S}$, }
    \State$\pi_{k+1}(s)=\arg_a\sum_{X_i}p(X_i)(\sum_{s'}X_i(s'|s,a)\cdot$\\
        $\hspace{11em}[r(s,a,s')+\gamma V^{\pi_k}_{X_i}(s')])$
    \EndFor
    \State $V^{k+1}_{X_i}=$PolicyEvaluation$(\pi_{k+1})$
    \EndFor
    \State return $\pi_n$
    \end{algorithmic}
\end{algorithm}

\begin{proof}
  Consider a 3-state, 3-action MDP. The state space is $\mathcal{S}=\{0,1,2\}$, where the states are in a circle. The action space is $\mathcal{A}=\{left, \ right, stay\}$, with the discount factor $\gamma=0.9$. The ground-truth transition probability $X$ takes two possible values $\{X_1, X_2\}$ with the prior distribution $p(X_1)=p(X_2)=0.5$. Tables \ref{numerical-X1} and \ref{numerical-X2} give the transition kernels $X_1$ and $X_2$ with respect to every action in $\mathcal{A}$. The reward function $r(s,s')$ is displayed in Table \ref{reward-table}:
\begin{table}[h]
    \centering
            \centering
            \begin{tabular}{|c|c|c|}
                \hline
                0 & 0 & 1 \\
                \hline
                1 & 0 & 0 \\
                \hline
                0 & 1 & 0 \\
                \hline
            \end{tabular}
            \begin{tabular}{|c|c|c|}
                \hline
                0 & 1 & 0 \\
                \hline
                0 & 0 & 1 \\
                \hline
                1 & 0 & 0 \\
                \hline
            \end{tabular}
            \begin{tabular}{|c|c|c|}
                \hline
                1 & 0 & 0 \\
                \hline
                0 & 1 & 0 \\
                \hline
                0 & 0 & 1 \\
                \hline
            \end{tabular}
            \caption{$X_{1}^{left},\ X_1^{right},\ X_1^{stay}$}
            \label{numerical-X1}
        \end{table}
        \begin{table}[h]
            \centering
            \begin{tabular}{|c|c|c|}
                \hline
                0 & 1 & 0 \\
                \hline
                0 & 0 & 1 \\
                \hline
                1 & 0 & 0 \\
                \hline
            \end{tabular}
            \begin{tabular}{|c|c|c|}
                \hline
                1 & 0 & 0 \\
                \hline
                0 & 1 & 0 \\
                \hline
                0 & 0 & 1 \\
                \hline
            \end{tabular}
            \begin{tabular}{|c|c|c|}
                \hline
                0 & 0 & 1 \\
                \hline
                1 & 0 & 0 \\
                \hline
                0 & 1 & 0 \\
                \hline
            \end{tabular}
            \caption{$X_{2}^{left},\ X_2^{right},\ X_2^{stay}$}
            \label{numerical-X2}
\end{table}
\begin{table}[h]
            \centering
            \begin{tabular}{|c||c|c|c|}
                \hline
                $s \downarrow \ s' \rightarrow$ & 0 & 1 & 2 \\
                \hline
                0 & 0.06 & 0.1 & 0.15 \\
                \hline
                1 & 0.02 & 0.1 & 0.15 \\
                \hline
                2 & 0.01 & 0.2 & 0.95 \\
                \hline
            \end{tabular}
            \caption{Reward setting of the planning}
            \label{reward-table}
\end{table}

This MDP with a random kernel has an optimal solution $\pi^*(\cdot|Y)=\{0:left, 1:left, 2:stay\}$, found by exhaustive search. The optimal solution associated with the fixed transition kernel $X_1$ is $\{0:left, 1:right, 2:stay\}$, and the optimal solution associated with the fixed kernel $X_2$ is $\{0:stay, 1:left, 2:right\}$. Notice that optimal solution $\pi^*(\cdot|Y)$ is neither equal to the optimal policy associated with the fixed $X_1$, nor $X_2$, i.e., $\pi^{*}(\cdot|Y)\neq \pi^*(X_1)$, and $\pi^*(\cdot|Y)\neq\pi^*(X_2)$.

    Notice that in our setting, the three probability kernels in $X_2$ associated with actions are the permutation of probability kernels in $X_1$, i.e., $X_2^{left}=X_{1}^{right}$, $X_2^{right}=X_{1}^{stay}$, and $X_2^{stay}=X_{1}^{left}$. Let $r_{X_i}(s,a)=\sum_{s'}X_i(s'|s,a)r(s,a,s')$ denote the immediate reward gained under the fixed $X_i$. We have that the largest immediate reward gained under $X_1$ is the same as the largest immediate reward gained under $X_2$, $\max_{a\in\mathcal{A}} r_{X_1}(s,a)=\max_{a\in\mathcal{A}} r_{X_2}(s,a),\forall s$.

    We claim that with the prior probability $p(X_1)=p(X_2)=1/2$, and with the initial policy $\pi_0=\pi^*(X_1)$, Algorithm \ref{policy iteration} fails to converge to $\pi^*$.
    
    Starting with $\pi_0=\pi^*(X_1)$, then by policy evaluation, we have that the updated V-values with fixed kernel $X_1$ by following $\pi_0$ is $V^{\pi_0}_{X_1}=\{0:7.858, 1:7.858, 2:8.658\}$, i.e., $V^{\pi_0}_{X_1}(0)=7.858$, $V^{\pi_0}_{X_1}(1)=7.858$ and $V^{\pi_0}_{X_1}(2)=8.658$. The V-value with fixed kernel $X_2$ following $\pi_0$ is updated to
    $V^{\pi_0}_{X_2}=\{0:0.911, 1:0.911, 2:1.02\}$. Notice that $V^{\pi_0}_{X_1}(s)\geq V^{\pi_0}_{X_2}(s), \forall s$.
    The extended policy iteration updates $\pi_1(s)$ such that:
    \begin{align*} \pi_1(s)&=\arg\max_{a}1/2\sum_{s'}X_1(s'|s,a)[r(s,a,s')+\gamma V^{\pi_0}_{X_1}(s')] 
    + 1/2\sum_{s'}X_2(s'|s,a)[r(s,a,s')+\gamma V^{\pi_0}_{X_2}(s')]\\
    &=\arg\max_a 1/2(r_{X_1}(s,a) + \gamma\sum_{s'} X_1(s'|s,a)V^{\pi_0}_{X_1}(s'))
     1/2(r_{X_2}(s,a) + \gamma\sum_{s'} X_2(s'|s,a)V^{\pi_0}_{X_2}(s'))\\
    &=\arg\max_{a} 1/2 Q^{\pi_0}_{X_1}(s,a) + 1/2Q^{\pi_0}_{X_2}(s,a)\\
    \end{align*}
    The extended Q-values $Q^{\pi_0}(s,a)=1/2Q^{\pi_0}_{X_1}(s,a) + 1/2Q^{\pi_0}_{2}(s,a)$ are shown in Table \ref{Q-values}.
    Since for all states $\{0,1,2\}$, we have that $\pi_1(0)=\arg\max_aQ^{\pi_0}(0,a)=0, \pi_1(1)=\arg\max_aQ^{\pi_0}(1,a)=1,$ and $\pi_1(2)=\arg\max_aQ^{\pi_0}(2,a)=2$. The updated policy has no change: $\pi_1=\pi_0=\pi^*(X_1)$. Therefore, the evaluations of $V^{\pi_1}_{X_i}(s)$ stay the same as $V^{\pi_0}_{X_i}(s),\forall s$.
    The update of the following sequence of policies $\{\pi_2,\pi_3,\dots,\pi_n\}$ repeats the above process, and are all equal to the initial policy, i.e. $\pi_n=\pi_0=\pi^*(X_1)$. Therefore, the algorithm \ref{policy evaluation} and \ref{policy iteration} fails to converge to $\pi^*(\cdot|Y)$.

\begin{table}[h]
        \centering
        \begin{tabular}{|c||c|c|c|}
        \hline
            $s\downarrow a\rightarrow$ & left&right&stay \\
            \hline
             0&4.43&4.02&4.10\\
             \hline
             1&4.08&4.43&4.01\\
             \hline
             2&4.05&4.47&4.88\\
             \hline
        \end{tabular}
        \caption{$Q^{\pi_0}(s,a)$}\label{Q-values}
        \label{tab:my_label}
    \end{table} 
\end{proof}

\section{Random  attacks on model-free RL, and extensions to random state observation attacks}
Previous sections address the random ``invincible'' attacks in planning. We want to generalize the scenario of applying the random ``invincible adversarial attack'' on model-free algorithms such as tabular and deep Q-learning (DQN). 

In Section \ref{sec:formulation}, we mentioned that the rate-distortion attack on transition kernels can be performed by randomly changing state observations. In case of a finite discrete transition kernel, the attacker can randomly map the ground-truth state to the delusional state by choosing a random permutation. When the state space is large or continuous and the transition kernel is implicitly deterministic, e.g. in Atari games or Cartpole games, the attacker can also perform the information-theoretic random state observation attacks. Specifically, let $\delta \sim \Omega(\delta)$ be a random perturbation sampled from the designed distribution $p(\delta)$. At each iteration within every episode, the attackers perturb the agent’s true state observation by adding the perturbation $\delta$. This perturbation $\delta$ remains constant throughout the episode once it is sampled from $p(\delta)$. Notice that the perturbed state observation is random due to the randomness of $\delta$. The victim agent then finds a false policy given the perturbed observation in learning.

\section{Numerical Analysis}
\subsection*{Regret analysis of rate-distortion attacks on planning with $I(X;Y)=0$.}\label{numerical:01} 
We first describe an effective rate-distortion attack in the planning. We consider the same MDP setting as in the proof of Theorem \ref{thm5.1}.

In terms of attack setting, we let the sample space of the delusional transition kernel $\Omega(Y)$ be the same as $\Omega(X)$ for the sake of presentation. The likelihoods $P(Y|X)$ that describe the probability that the victim observes the delusional matrix $Y$ given the ground truth $X$, are given in Table \ref{rate-distortion-attack}.

\begin{table}[h]
\centering
    \begin{minipage}{0.4\columnwidth}  
        \begin{tabular}{|c||c|c|}
            \hline
            $X \downarrow Y \rightarrow$ & $X_1$ & $X_2$ \\
            \hline
            $X_1$ & $\ 1/2\ $ & $\ 1/2$ \ \\
            \hline
            $X_2$ & $\ 1/2\ $ & $\ 1/2\ $ \\
            \hline
        \end{tabular}
        
        \caption{Attack parameters $P(Y|X)$}
        \label{rate-distortion-attack}
    \end{minipage}%
    \hspace{0.05\columnwidth} 
    \begin{minipage}{0.4\columnwidth}  
        
        \begin{tabular}{|c||c|c|}
            \hline
            $X \downarrow Y \rightarrow$ & $X_1$ & $X_2$ \\
            \hline
            $X_1$ & $1 - p_1$ & $p_1$ \\
            \hline
            $X_2$ & $p_2$ & $1 - p_2$ \\
            \hline
        \end{tabular}
        
        \caption{Minimum mutual information attack}
        \label{mi-attack-table}
    \end{minipage}
\end{table}
With the prior probability $p(X_1)=p(X_2)=1/2$, the posterior probabilities $p(X|Y)$ perceived by the victim agent are: $p(X=X_i|Y=X_j)=1/2,\forall i,j\in\{1,2\}$.

We compute the optimal policy $\pi^*(\cdot|Y)$ using this specific-designed rate-distortion attack with an exhaustive search of all possible policies. The results show that the optimal policy $\pi^*(\cdot|Y)$ for the recipient agent is a ``intermediate policy'' between the two optimal policies: $\pi^*(X_1)$ and $\pi^*(X_2)$. The optimal policy under the rate distortion attack is $\pi^*(\cdot|Y_1)=\pi^*(\cdot|Y_2)=\{0:left, 1:left, 2:stay\}$, which means $\pi^*(\cdot|Y_i)(0)=left, \pi^*(\cdot|Y_i)(1)=left$ and $\pi^*(\cdot|Y_i)(2)=stay, i=1,2$. 

For fixed kernels, $\pi^*(X_1)=\{0:left, 1:right, 2:stay\}$, while $\pi^*(X_2)=\{0:stay, 1:left, 2:right\}$. We compute regret $R=\mathbb{E}_{p(X,Y)}\|V^{\pi^*(X)}_X-V^{\pi^*(\cdot|Y)}_X\|_{\infty}$ and it turns out that $R = 3.84$. The regret takes approximately $44.3\%$ of the real optimal V-value $\mathbb{E}_{p(X)}\|V^{\pi^*(X)}_X\|_{\infty}$ without attack.

\subsection*{Optimal Budget constrained attack maximizing the regret}\vspace{-0.25cm}
We extend our numerical results for rate-distortion attacks when attackers are limited by the attack budget $B$. Assume that there is no cost for the attackers to confuse the agent to perceive $X_1$ as $X_1$, or $X_2$ as $X_2$. The budget is spent on confusing agents to perceive $X_1$ as $X_2$, and to perceive $X_2$ as $X_1$. We denote the probability $p_1 = P(Y=X_2|X=X_1)$, which is the likelihood that the ground-truth kernel $X_1$ leads to a delusional kernel $X_2$. Similarly, $p_2 = P(Y=X_1|X=X_2)$ is the likelihood that the ground-truth kernel $X_2$ leads to a delusional kernel $X_1$. The attacker's goal is to design $p_1$ and $p_2$ such that regret is maximized while not exceeding budget $B$.

Let the cost of achieving the transition $X_1 \rightarrow X_2$ be $p_1c_1,\  c_1=1.5$, and let the cost of achieving the transition $X_2 \rightarrow X_1$  be $p_2c_2,\ c_2 = 2$. We require the average cost to be less than budget B, namely, $0.5p_1c_1+0.5p_2c_2<B$.

With a sequence of budgets, we record the maximum regret caused by the most damaging attack (optimizing $p_1$ and $p_2$) with the limit of every budget in Figure \ref{fig:budget-constrained-regret}. As shown in Figure \ref{fig:budget-constrained-regret}, the regret keeps increasing as the budget increases. When the budget reaches $0.711$, the regret reaches its maximum for the specific MDP environment, which is about $44.3\%$ of the $\ell_\infty$ norm of the unattacked V-value.

\begin{figure}[htbp]
    \centering
    \begin{minipage}[t]{0.49\columnwidth}
        \centering
        \includegraphics[width=\linewidth]{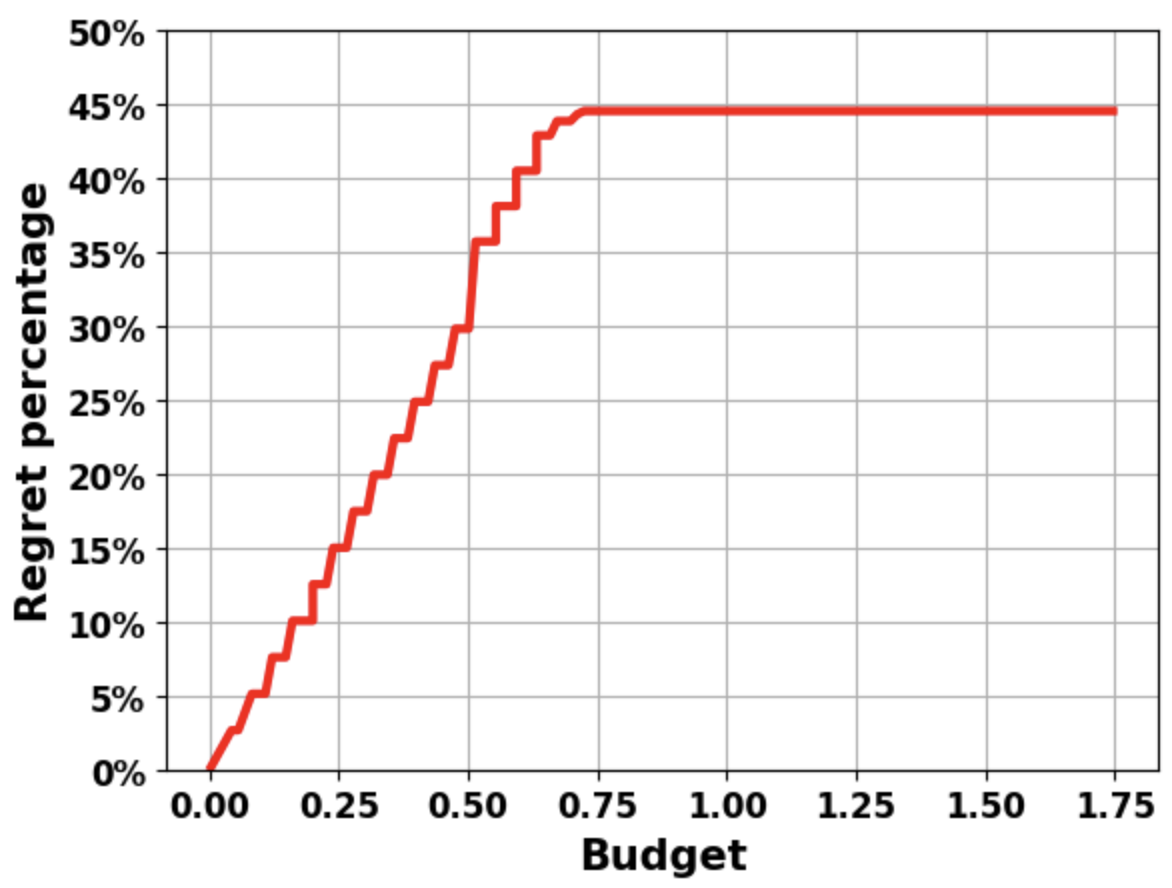}
        \caption{Budget constrained optimal attack maximizing regret}
        \label{fig:budget-constrained-regret}
    \end{minipage}
    \hfill
    \begin{minipage}[t]{0.49\columnwidth}
        \centering
        \includegraphics[width=\linewidth]{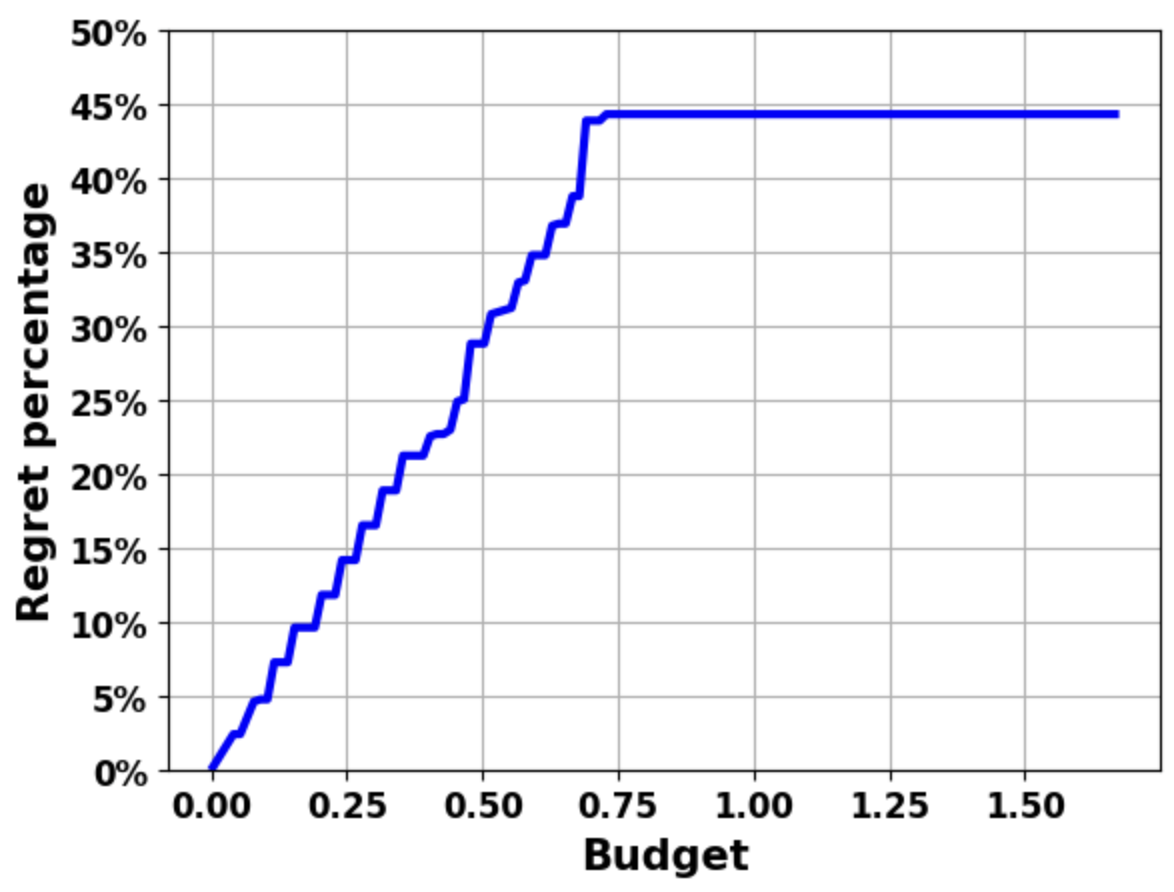}
        \caption{Minimized mutual information attack under budget}
        \label{fig:MI-attack}
    \end{minipage}
\end{figure}

\subsection*{Rate-distortion attacks via minimizing mutual information}
In this experiment, we consider our proposed rate distortion attack by directly minimizing mutual information between the ground-truth kernel and the delusional transition kernel. The mutual information attack prevents the victim agent from gaining information from the observed delusional transition about the ground-truth transitional kernel $X$ as much as possible. 

By optimizing the parameters $(p_1,p_2)$ for $P(Y|X)$ displayed in Table \ref{mi-attack-table}, we can build a minimum mutual information attack. In Figure \ref{MI}, under the budget $B$, we calculate the minimum mutual information:$\min I(X; Y)=\min_{p_1,p_2} [-0.5*(1+p_2-p_1)*\log(0.5+0.5*(p_2-p_1))-0.5*(1+p_1-p_2)*\log(0.5+0.5*(p_1-p_2))+0.5*( p_1*\log(p_1)+(1-p_1)*\log(1-p_1) )+0.5*( p_2*\log(p_2)+(1-p_2)*\log(1-p_2) )]$
 , under the constraint that $0.5p_1c_1+0.5p_2c_2<B.$ 

The effect of the minimum mutual information attack is shown in Figure \ref{fig:MI-attack}. We see that when the budget reaches $0.7285$, the regret reaches its maximum $44.3\%$.
\begin{figure}[h!]
    \centering
    \begin{minipage}[t]{0.5\textwidth}
        \centering
        \includegraphics[width=7cm,height=5.7cm]{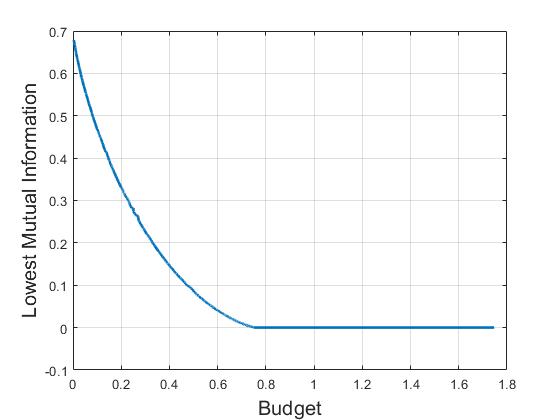}
    \caption{Minimum mutual information given budget}
    \label{MI}
    \end{minipage}%
    \hfill
    \begin{minipage}[t]{0.5\textwidth}
        \centering
    \includegraphics[width=7cm,height=5.5cm]{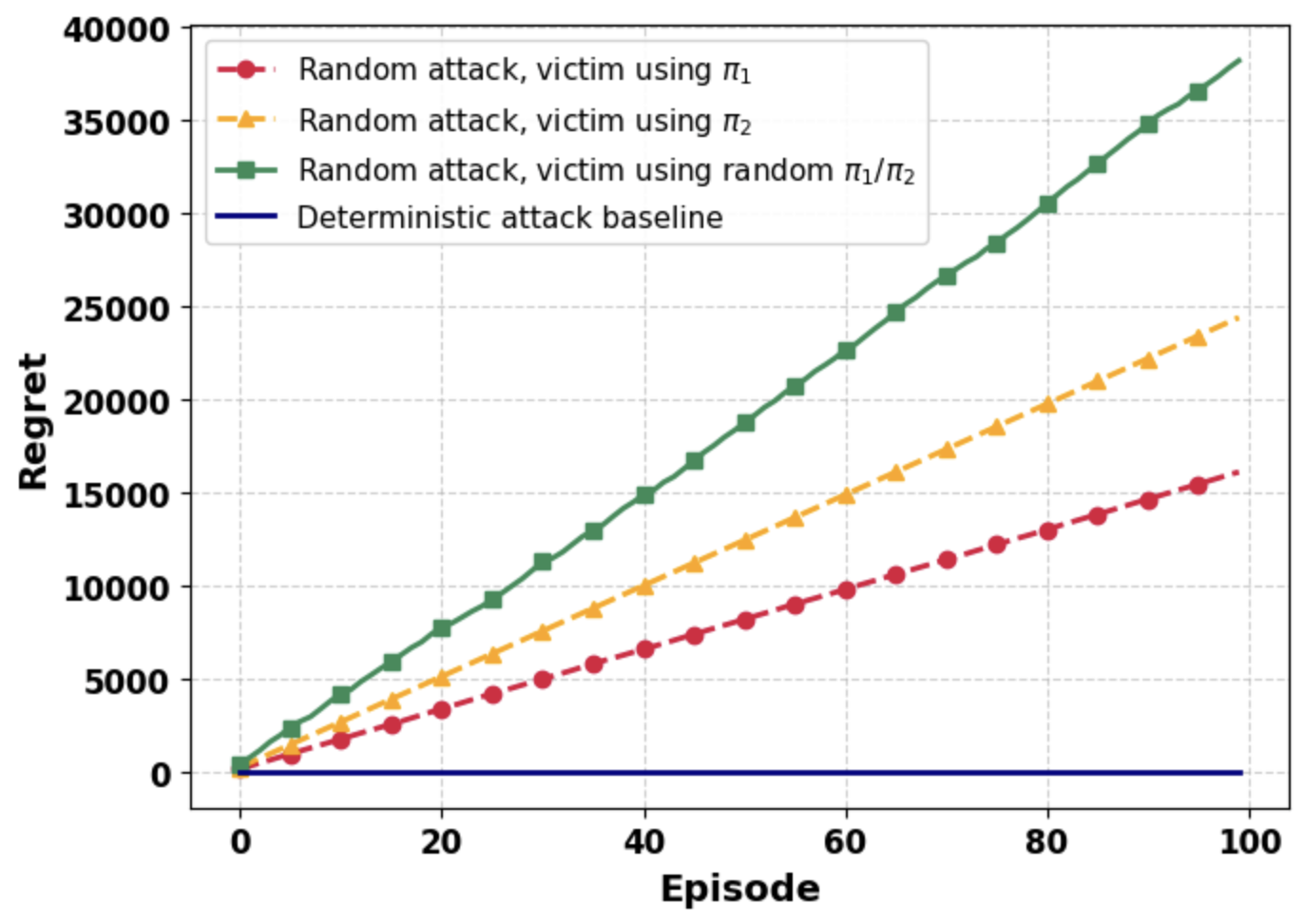}
    \caption{Regret of rate-distortion attacks on DQN}
    \label{DQN-attack}
    \end{minipage}
\end{figure}
\subsection*{Rate-distortion attacks on model-free RLs}
\textbf{(1) Rate-distortion attacks on tabular Q-learning in Block-world}: We tested the effect of random attacks on transition functions in the block world environment \cite{RusselNorvig,qlearningblockworld}. The state space of the block world environment is shown in Figure \ref{blockworld}. The settings of the Blockworld environment are from \cite{qlearningblockworld}, where the state space is a grid consisting of 12 distinct states. The green and red spots are terminal states: the green state represents a win with a reward of $+1$, while the red state represents a loss with a reward of $-1$. Each time the agent takes an action, it incurs a small penalty of $-0.04$. The black cell represents a wall. If the agent attempts to move into the wall or outside the grid boundaries, it will remain in its current position. The action space consists of \{east, west, north, south\}.
\begin{figure}[h!]
    \centering
    \begin{minipage}[t]{0.5\textwidth}
        \centering
        \includegraphics[width=7cm,height=4.5cm]{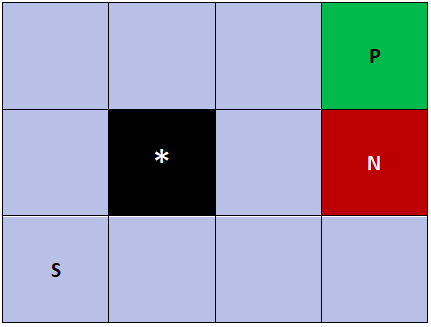}
        \caption{Block world environment}
        \label{blockworld}
    \end{minipage}%
    \hfill
    \begin{minipage}[t]{0.5\textwidth}
        \centering
        \includegraphics[width=7cm,height=4.5cm]{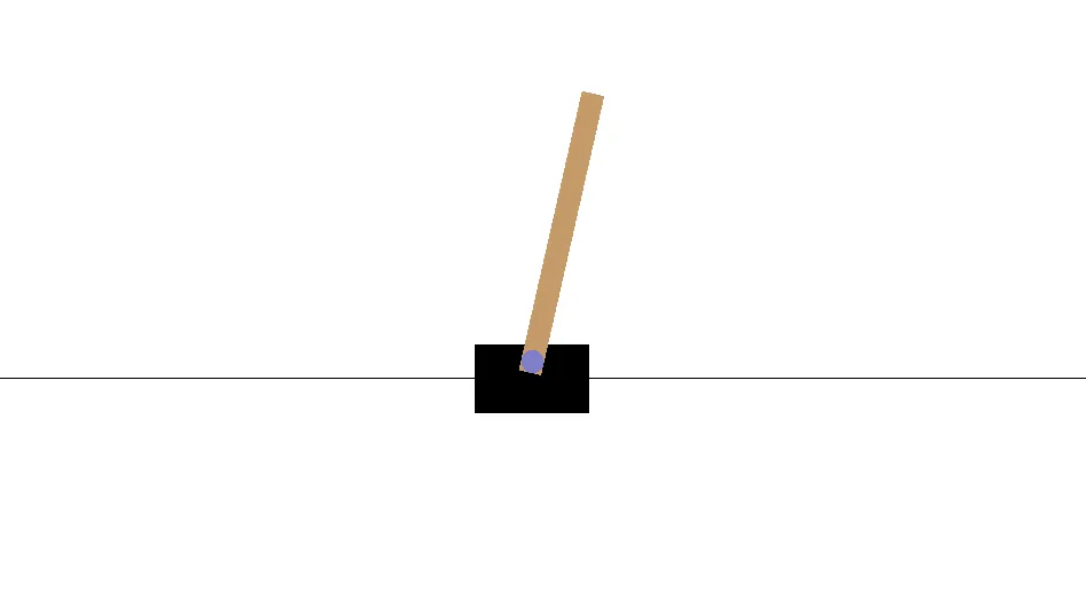}
        \caption{Cartpole}
        \label{cartpole}
    \end{minipage}
\end{figure}
We provide a detailed discussion of the random kernel settings and the adversarial attack on the block world environment.

\emph{Changing hyper-parameters}: The attackers create a fake training environment dynamics by changing the hyper-parameter ``slip probability'' $\alpha$. In block-world, when the RL agent chooses an action, there is a probability $\alpha$ that the agent will slip and end up in a different state. Specifically, if the agent slips, it may move in a direction orthogonal to the target direction. For example, if $\alpha= 0.8$ and the agent intends to go east, it will actually go east with probability $0.8$, go north with probability $0.1$, and go south with probability $0.1$. The slip probability $\alpha$ implicitly determines the state transitions in the block-world environment. We assume that $\alpha$ is random from a prior distribution where $p(\alpha=0.8)=p(\alpha=0.2)=1/2$. Let $X_1$ denote the transition kernel corresponding to $\alpha=0.8$, and $X_2$ the kernel for $\alpha=0.2$. To provide minimum information, we inherit the designed rate distortion attack parameters $p(X=X_i|Y=X_j)=1/2,i,j=1,2$. 

The regret of every starting state caused by the rate-distortion information-theoretic attack in the test is displayed in Figure \ref{TQ-learning}. We compare that with the regrets caused by a deterministic attack, which the agent can reverse, and find that the proposed random attack causes more loss. The regret of the proposed attack is computed assuming that the victim agent uses the false policies, for example $\pi^*(Y_1)$ or $\pi^*(Y_2)$, which were learned in the tabular Q-learning.
\begin{figure*}
\centering
    \includegraphics[width=17cm,height=32cm, keepaspectratio]{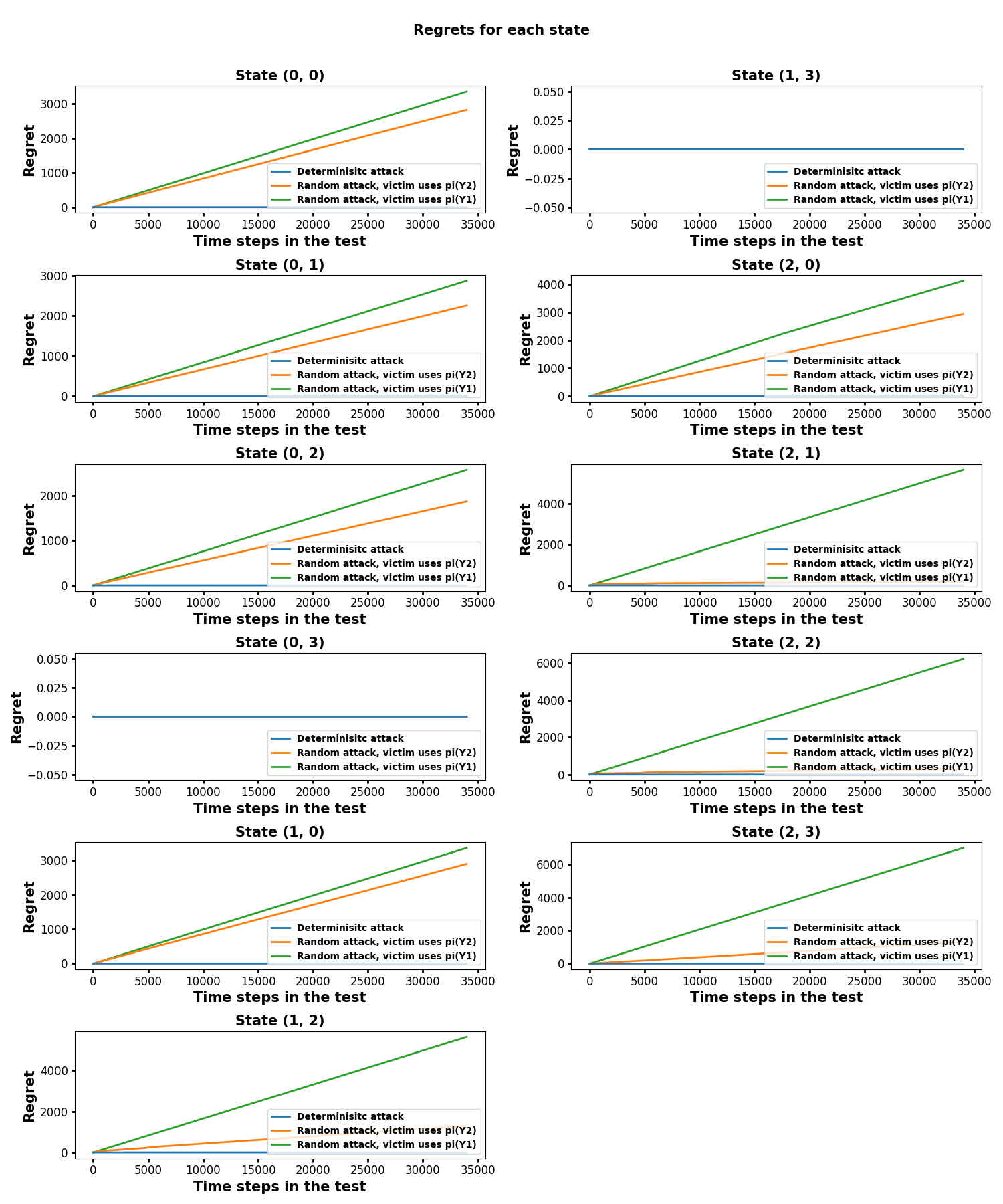}
    \caption{Regret of rate-distortion attacks on tabular Q-learning}
    \label{TQ-learning}
\end{figure*}

\textbf{(2) Rate distortion attacks on deep Q-learning in Cartpole}: We tested the effect of random attacks on the DQN algorithm in Cartpole \cite{brockman2016openai}. The exact settings of Cartpole-v1 environment are from \cite{gymnasium}. The agent is simulated as a pole attached by an un-actuated joint to a cart, as shown in Figure \ref{cartpole}, which moves along a frictionless track. The pendulum is placed upright on the cart and the goal is to balance the pole by applying forces in the left and right directions on the cart. The agent has to decide between two actions: moving the cart left or right, so that the pole attached to it stays upright. The state space is of 4 elements in Cartpole, with the values corresponding to the following positions and velocities: \{Cart Position, Cart Velocity, Pole Angle, Pole Angular Velocity\}. For the details of deterministic transition functions and other details, please refer to the documentation of \cite{gymnasium}. 

\textbf{Random state observation attack:} The attackers create fake transition dynamics by injecting perturbation noise $\delta$ into the state observations of the victim agent during training. The false disturbance noise $\delta_i$ is sampled from a designed perturbation distribution $p(\delta_i|\delta_j)$, where $\delta_i$ is the false noise while $\delta_j$ is the true noise in the training. Once $\delta_i$ is sampled, it remains constant during the DQN training. We tested a discrete perturbation distribution and every sampled noise $\delta_i$ corresponds to a different transition dynamics $X_i$ and an optimal policy $\pi^*(X_i)$. In our experiment, we considered a sample space of discrete perturbations $\{\delta_1 = (-0.3, 0.3, -0.3, 0.3), \delta_2=(0,0,0,0)\}$, and the designed perturbation distribution is $p(\delta_i|\delta_j)=1/2,i,j=1,2.$
The regret caused by the proposed random observation attack is shown in Figure \ref{DQN-attack}.

\textbf{(3) Random state observation attack via random permutation}: We show that our proposed information-theoretic rate-distortion attack can be implemented by perturbing the agent's state observations during training. In this experiment, the environment consists of $3$ discrete states and $3$ actions, For each fixed action, the random transition kernel can be represented by a  $3\times3$ transition matrix. There are $6$ possible transition kernels $\{X_1,X_2,\dots,X_6\}$. One example, $X_1$, is presented in Table \ref{model-based-X1}. The remaining $5$ transition kernels are obtained by permuting the rows and columns of $X_1$, that is, $X_i=P_iX_1P^T_i,i\in\{1,2,\dots,6\}$, where $P_i$ denotes the corresponding permutation matrix. We assume a uniform prior distribution over the $6$ true transition kernels. The reward function $r(s,a)$ used in this example is specified in Table \ref{tab:reward-model-based}.
\begin{table}[]
    \centering
    \begin{tabular}{|c||c|c|c|}
    \hline
         $s\downarrow a\rightarrow$&left&right&stay\\
         \hline
         $0$&$1.0$&$1.3$&$1.1$\\
         \hline
         $1$&$2.0$&$2.3$&$2.1$\\
         \hline
         $2$&$3.0$&$3.3$&$3.1$\\
         \hline
    \end{tabular}
    \caption{Reward table of the model-based training}
    \label{tab:reward-model-based}
\end{table}

\begin{table}[h]
    \centering
            \centering
            \begin{tabular}{|c|c|c|}
                \hline
                1 & 0 & 0 \\
                \hline
                0.8 & 0.2 & 0 \\
                \hline
                0 & 0.8 & 0.2 \\
                \hline
            \end{tabular}
            \begin{tabular}{|c|c|c|}
                \hline
                0.2 & 0.8 & 0 \\
                \hline
                0 & 0.2 & 0.8 \\
                \hline
                0 & 0 & 1 \\
                \hline
            \end{tabular}
            \begin{tabular}{|c|c|c|}
                \hline
                0.9 & 0.1 & 0 \\
                \hline
                0.1 & 0.8 & 0.1 \\
                \hline
                0 & 0.1 & 0.9 \\
                \hline
            \end{tabular}
            \caption{$X_{1}^{left},\ X_1^{right},\ X_1^{stay}$}
            \label{model-based-X1}
        \end{table}

Since there are $3$ states, the number of possible permutations between states is $3!=6$. A permutation over the states corresponds to simultaneously permuting the rows and columns of the transition kernel. For example, if the attacker maps the true state $0$ to the delusional state $1$ while keeping the true states $1$ and $2$ unchanged during training, this operation is equivalent to swapping both the first and second rows and the first and second columns of the true transition kernel matrix.

We assume that the victim agent does not know the exact transition kernels but only knows a uniform prior distribution over the set of possible ground-truth kernels. In each training episode, after the environment samples the true transition kernel, the agent interacts with the environment and collects a sequence of transition tuples $(s_1, a_1, r_1, s_2), (s_2, a_2, r_2, s_3),\dots, (s_T, a_T, r_T, s_{T+1})$, where $T=100$. Each empirical kernel entry $\hat{X}_{s,a}$ is estimated by the empirical frequency of state transitions. 

In each training episode, the attacker's strategy is randomly chosen from one of the six possible permutations. According to the chosen permutation, the attacker changes the true state observations by mapping each true state to its corresponding delusional state before the agent perceives it. The attacker will need to modify every state observation throughout the trajectory collection process. We show in Figure \ref{fig:state-obs-attack} the increase in regret under random state observation attacks across $20$ training episodes.
\begin{figure}
    \centering
    \includegraphics[width=0.5\linewidth]{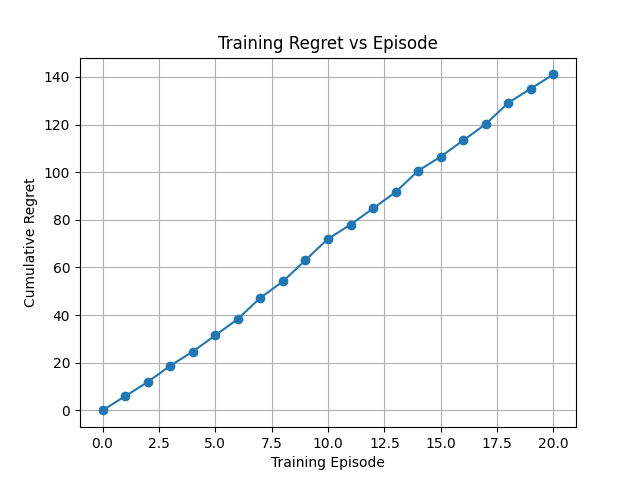}
    \caption{Training regret under rate-distortion state observation attacks}
    \label{fig:state-obs-attack}
\end{figure}
\section{Conclusion}
We propose an "invincible" type of adversarial attack: the rate-distortion information-theoretic attack. By forcing the victim agent to observe random transition kernels, the attack prevents it from gaining information about the ground-truth transition kernel. We prove a lower bound on the regret caused by the proposed attack and numerically show its impact on both planning and model-free learning. Both results show that this attack significantly decreases the agent's expected state values. We explore the existence of optimal policies and a planning algorithm for MDPs with random kernels. Please see more discussions on the existence of optimal policy in the Appendix. The limitations include the need to develop effective policies when the traditional optimal policy does not exist. In addition, identifying convergent planning/learning algorithms for these settings is future work.


\bibliographystyle{unsrt}  
\bibliography{refs}

\begin{thebibliography}{10}

\bibitem{ad}
Shalev-Shwartz, Shaked Shammah, and Amnon Shuashua.
\newblock Safe, multi-agent, reinforcement learning for autonomous driving.
\newblock In {\em arXiv preprint arXiv:1610.03295}, 2016.

\bibitem{fd}
Xiao-Yang Liu, Hongyang Yang, Qian Chen, Runjia Zhang, Liuqing Yang, Bowen Xiao, and Christina~Dan Wang.
\newblock Finrl: A deep reinforcement learning library for automated stock trading in quantitative finance.
\newblock In {\em Deep RL Workshop}. NeurIPS, 2020.

\bibitem{rs}
Xiangyu Zhao, Long Xia, Liang Zhang, Zhuoye Ding, and Jiliang Tang.
\newblock Deep reinforcement learning for page-wise recommendations.
\newblock In {\em The 12th ACM Conference on Recommender Systems}, pages 95--103. ACM, 2018.

\bibitem{wireless-commu}
Nicholas Mastronarde and Mihaela van~der Schaar.
\newblock Fast reinforcement learning for energy-efficient wireless communication.
\newblock {\em IEEE Transactions on Signal Processing}, 59(12):6262--6266, 2011.

\bibitem{drones}
Yunlong Song, Mats Steinweg, Elia Kaufmann, and Davide Scaramuzza.
\newblock Autonomous drone racing with deep reinforcement learning.
\newblock In {\em International Conference on Intelligent Robots and Systems (IROS)}. IEEE/RSJ, 2021.

\bibitem{att-costsig}
Yunhan Huang and Quanyan Zhu.
\newblock Deceptive reinforcement learning under adversarial manipulations on cost signals.
\newblock In {\em International Conference on Decision and Game Theory for Security}, pages 217--237. Springer, 2019.

\bibitem{davis-ma}
Guanlin Liu and Lifeng Lai.
\newblock Efficient adversarial attacks on online multi-agent reinforcement learning.
\newblock In {\em arXiv preprint arXiv:2307.07670}, 2023.

\bibitem{ada-rpa}
Xuezhou Zhang, Yuzhe Ma, Adish Singla, and Xiaojin Zhu.
\newblock Adaptive reward-poisoning attacks against reinforcement learning.
\newblock In {\em International Conference on Machine Learning}, volume 119, pages 11225--11234, 2020.

\bibitem{davis-sa}
Guanlin Liu and Lifeng Lai.
\newblock Provably efficient black-box action poisoning attacks against reinforcement learning.
\newblock In {\em Advances in Neural Information Processing Systems}, volume~34, pages 12400--12410, 2021.

\bibitem{batch-rl}
Yuzhe Ma, Xuezhou Zhang, Wen Sun, and Jerry Zhu.
\newblock Policy poisoning in batch reinforcement learning and control.
\newblock In {\em Advances in Neural Information Processing Systems}, volume~32, 2019.

\bibitem{spaold}
Michael Littman.
\newblock Markov games as a framework for multi-agent reinforcement learning.
\newblock In {\em Proceedings of the Eleventh International Conference}, 1994.

\bibitem{envattack1}
Hang Xu, Xinghua Qu, and Zinovi Rabinovich.
\newblock Policy resilience to environment poisoning attacks on reinforcement learning.
\newblock In {\em arXiv preprint arXiv:2304.12151}, 2023.

\bibitem{poli-induc}
Vahid Behzadan and Arslan Munir.
\newblock Vunerability of deep reinforcement learning to policy induction attacks.
\newblock In {\em International Conference on Machine Learning and Data Mining in Pattern Recogonition}, pages 262--275. Springer, 2017.

\bibitem{iyengar2005robust}
Garud Iyengar.
\newblock Robust dynamic programming.
\newblock {\em Mathematics of Operations Research}, 30(2):257--280, 2005.

\bibitem{nilim2003robustness}
Arnab Nilim and Laurent El~Ghaoui.
\newblock Robustness in markov decision problems with uncertain transition matrices.
\newblock In {\em Advances in Neural Information Processing Systems}, volume~16. MIT Press, 2003.

\bibitem{pinto2017robust}
Lerrel Pinto, James Davidson, Rahul Sukthankar, and Abhinav Gupta.
\newblock Robust adversarial reinforcement learning.
\newblock {\em arXiv preprint arXiv:1703.02702}, 2017.

\bibitem{xu2010distributionally}
Huan Xu and Shie Mannor.
\newblock Distributionally robust markov decision processes.
\newblock In {\em Advances in Neural Information Processing Systems (NeurIPS)}, volume~23, 2010.

\bibitem{duff2002optimal}
Michael~O'Gordon Duff.
\newblock {\em Optimal Learning: Computational Procedures for Bayes-Adaptive Markov Decision Processes}.
\newblock Ph.d. dissertation, University of Massachusetts Amherst, 2002.

\bibitem{ghavamzadeh2016bayesian}
Mohammad Ghavamzadeh, Shie Mannor, Joelle Pineau, and Aviv Tamar.
\newblock Bayesian reinforcement learning: A survey.
\newblock {\em Foundations and Trends in Machine Learning}, 8(5–6):359--483, 2015.
\newblock preprint arXiv:1609.04436.

\bibitem{lin2022bayesian}
Yifan Lin, Yuxuan Ren, and Enlu Zhou.
\newblock Bayesian risk markov decision processes.
\newblock In {\em Advances in Neural Information Processing Systems 35 (NeurIPS 2022)}, 2022.

\bibitem{rakhsha2020policyteachingenvironmentpoisoning}
Amin Rakhsha, Goran Radanovic, Rati Devidze, Xiaojin Zhu, and Adish Singla.
\newblock Policy teaching via environment poisoning: Training-time adversarial attacks against reinforcement learning, 2020.

\bibitem{franzmeyer2024illusory}
Tim Franzmeyer, Stephen~Marcus McAleer, João~F. Henriques, Jakob~N. Foerster, Philip H.~S. Torr, Adel Bibi, and Christian Schroeder~de Witt.
\newblock Illusory attacks: Information‑theoretic detectability matters in adversarial attacks.
\newblock In {\em Proceedings of the 12th International Conference on Learning Representations (ICLR 2024)}, 2024.
\newblock Spotlight presentation.

\bibitem{game-theoretic}
Jing Yu, Clement Gehring, Florian Sch\"afer, and Animashree Anandkumar.
\newblock Robust reinforcement learning: A constrained game-theoretic approach.
\newblock In Ali Jadbabaie, John Lygeros, George~J. Pappas, Pablo A.~\&~Parrilo, Benjamin Recht, Claire~J. Tomlin, and Melanie~N. Zeilinger, editors, {\em Proceedings of the 3rd Conference on Learning for Dynamics and Control}, volume 144 of {\em Proceedings of Machine Learning Research}, pages 1242--1254. PMLR, 07 -- 08 June 2021.

\bibitem{DBLP:journals/corr/abs-1712-03632}
Anay Pattanaik, Zhenyi Tang, Shuijing Liu, Gautham Bommannan, and Girish Chowdhary.
\newblock Robust deep reinforcement learning with adversarial attacks.
\newblock {\em CoRR}, abs/1712.03632, 2017.

\bibitem{infobook}
Thomas~M. Cover and Joy~A. Thomas.
\newblock {\em Elements of Information Theory (Wiley Series in Telecommunications and Signal Processing)}.
\newblock Wiley-Interscience, USA, 2006.

\bibitem{sutton-book}
Richard~S. Sutton and Andrew~G. Barto.
\newblock {\em Reinforcement Learning: An Introduction}.
\newblock MIT Press, 2 edition, 2018.

\bibitem{RusselNorvig}
Stuart Russell and Peter Norvig.
\newblock {\em Artificial Intelligence: A Modern Approach, 4th US ed.}
\newblock USA, 2022.

\bibitem{qlearningblockworld}
Senad Kurtisi.
\newblock Q-learning block world.
\newblock https://github.com/senadkurtisi/Q-learning-block-world, 2021.

\bibitem{brockman2016openai}
Greg Brockman, Vicki Cheung, Ludwig Pettersson, Jonas Schneider, John Schulman, Jie Tang, and Wojciech Zaremba.
\newblock Openai gym.
\newblock \url{https://arxiv.org/abs/1606.01540}, 2016.

\bibitem{gymnasium}
{Farama Foundation}.
\newblock Gymnasium: A toolkit for developing and comparing reinforcement learning algorithms.
\newblock \url{https://gymnasium.farama.org}, 2022.
\newblock Accessed: 2025-05-15.

\end{thebibliography}

\end{document}